\documentclass{article}

\usepackage{microtype}
\usepackage{graphicx}
\usepackage{subfigure}
\usepackage{booktabs} 

\usepackage{hyperref}
\usepackage{algorithm2e}

\usepackage[accepted]{icml2023}

\usepackage{amsmath}
\usepackage{amssymb}
\usepackage{thm-restate}
\usepackage{mathtools}
\usepackage{amsthm}

\usepackage[capitalize,noabbrev]{cleveref}

\theoremstyle{plain}
\newtheorem{thm}{Theorem}[section]
\newtheorem{prop}[thm]{Proposition}
\newtheorem{lem}[thm]{Lemma}

\theoremstyle{definition}

\theoremstyle{remark}

\usepackage[textsize=tiny]{todonotes}

\DeclareMathOperator*{\argmax}{arg\,max}
\DeclareMathOperator*{\argmin}{arg\,min}

\DeclareMathOperator{\KL}{KL}
\DeclareMathOperator{\Diag}{Diag}

\newcommand{\cC}{\mathcal{C}}
\newcommand{\cR}{\mathcal{R}}
\newcommand{\cT}{\mathcal{T}}
\newcommand{\R}{\mathbb{R}}

\newcommand{\ROUND}{\rho}
\newcommand{\radius}{\delta}
\newcommand{\Esp}[2]{\mathbb{E}_{#1}\left[ #2 \right]}

\icmltitlerunning{Mirror Sinkhorn: Fast Online Optimization on Transport Polytopes}

\begin{document}

\twocolumn[
\icmltitle{Mirror Sinkhorn: Fast Online Optimization on Transport Polytopes}



\begin{icmlauthorlist}
\icmlauthor{Marin Ballu}{camb}
\icmlauthor{Quentin Berthet}{goog}
\end{icmlauthorlist}

\icmlaffiliation{goog}{Google DeepMind, Paris, FR}
\icmlaffiliation{camb}{University of Cambridge, UK}

\icmlcorrespondingauthor{Quentin Berthet}{qberthet@google.com}

\icmlkeywords{Machine Learning, ICML}

\vskip 0.3in
]



\printAffiliationsAndNotice{}  

\begin{abstract}
    Optimal transport is an important tool in machine learning, allowing to capture geometric properties of the data through a linear program on transport polytopes. We present a single-loop optimization algorithm for minimizing general convex objectives on these domains, utilizing the principles of Sinkhorn matrix scaling and mirror descent. The proposed algorithm is robust to noise, and can be used in an online setting. We provide theoretical guarantees for convex objectives and experimental results showcasing it effectiveness on both synthetic and real-world data.
\end{abstract}
\section{Introduction}
Optimal transport is a seminal problem in optimization \cite{monge1781memoire}, and an important topic in analysis \cite{villani2008optimal}. The discrete case is a linear program on the set of nonnegative matrices with fixed row and column sums \cite{kantorovich1942translocation}. This set forms the {\em transport polytope}, whose elements can be interpreted as the law of joint distributions for $(X, Y)$ with marginal distributions $\mu, \nu$.

In machine learning, OT has recently gained in importance following the work of \citet{cuturi2013sinkhorn}, presenting an entropic-regularized method and using the Sinkhorn algorithm to efficiently optimize it \citep[see, e.g.][and references therein for an overview of this topic and its applications]{peyre2019computational}. This method can also be used to solve the original OT problem \cite{altschuler2017near}, by setting a regularization parameter that depends on the cost matrix and desired precision level. We consider here the more general problem of convex optimization of \textbf{any convex function} on the transport polytope. This comprises several optimization problems, included but \textbf{not limited} to both optimal transport and its entropic regularized version. The problem for a quadratic function has been studied both for the purpose of registering point cloud \cite{grave2019unsupervised} (related to Gromov-Wasserstein problems \cite{memoli2011gromov, solomon2016entropic}) and for computing euclidean projection on the Birkhoff polytope \cite{li2020efficient}. It appears in statistical inference on random permutations \cite{birdal2019probabilistic}. Inference on random permutations can be obtained by minimizing various other convex functions \cite{linderman2018reparameterizing}. Optimisation on this polytope also arises when trying to both compute and minimize a Wasserstein distance or sum of Wasserstein distances of a set of parametrized distributions, e.g. in computation of Wasserstein estimators \cite{ballu2020stochastic,bassetti2006minimum}, private learning \cite{boursier2019private}, Wasserstein barycenters \cite{rabin2011wasserstein,agueh2011barycenters,cuturi2014fast}, topology learning \cite{le2023refined}, experiment design \cite{berthet2015resource}, and generative models. In the latter case, the distribution generated by a neural network is compared to the sample distribution with the 1-Wasserstein distance in WGAN \cite{arjovsky2017wasserstein} and Wasserstein autoencoders \cite{tolstikhin2017wasserstein}, or with a regularized version of the Wasserstein distance with Sinkhorn divergences \cite{genevay2018learning}. 

Several algorithms that have been suggested to solve optimal transport use iterated Bregman projections \cite{benamou2015iterative, dvurechensky2018computational}. To minimize general convex functions on transport polytopes, we extend the approach to a single loop iterated algorithm. 
Interpretations of Sinkhorn algorithm as mirror descent in the dual \cite{mishchenko2019sinkhorn} and the primal \cite{leger2021gradient, aubin2022mirror} have been used to derive convergence rates and to extend Sinkhorn to this more general problem. These extensions are supported by an analysis, assuming smoothness and strong convexity of the objective. Streaming iterations of Sinkhorn have been proposed by \citet{mensch2020online}, focusing solely on minimizing a regularized optimal transport problem.
General convex optimization algorithms can theoretically achieve the best asymptotic rates for optimal transport. An $\varepsilon$-close optimal transport plan can be obtained with accelerated gradient descent schemes \cite{dvurechensky2018computational, lin2019efficient, guo2020fast}, or accelerated alternative minimisation \cite{guminov2021combination}. These are only analyzed when given deterministic gradient updates, and their parameters depend on an imposed desired level of optimization precision $\varepsilon$.

Using the Sinkhorn algorithm to solve OT, as studied by \citet{altschuler2017near} requires to set a regularization parameter $\alpha >0$, as a function of the desired optimization precision. Indeed, this algorithm is tied to an entropic-regularized version of this problem, whose solution $\gamma_\alpha^*$ is different: there is a regularization bias. We note that this has some advantages. In particular, the solution is a continuously differentiable function of the problem inputs \cite{peyre2019computational}. This is part of a wide effort to create differentiable versions of discrete operators such as optimizers \cite{cuturi2017soft, berthet2020learning, blondel2020fast, vlastelica2019differentiation, paulus2020gradient}, to ease their inclusion in end-to-end differentiable pipelines that can be trained with first-order methods in applications \cite{cordonnier2021differentiable, kumar2021groomed, carr2021self, le2021differentiable, baid2022deepconsensus, llinares2021deep} and other optimization algorithms \cite{dubois2022fast}. This regularized objective, as well as alternate regularizations \cite{blondel2018smooth} fall within our framework, and can be optimized using our algorithm.

Our algorithm, which we call {\em Mirror Sinkhorn}, is based on the principles of mirror descent on the transport polytope (which requires an oracle to solve Bregman-regularized linear problems on this set), and the Sinkhorn algorithm which enforces normalization of rows and columns to satisfy the constraints. The use of multiple Sinkhorn steps for solving a mirror descent oracle has been proposed in \cite{alvarez2018structured}, and the convergence of the resulting algorithm has been further analysed in \cite{xie2020fast} and later in \cite{aubin2022mirror} for smooth and strongly convex objectives. We provide an analysis for a single step of Sinkhorn normalization between gradient updates.

\paragraph{Our contributions.}
The \emph{Mirror Sinkhorn} algorithm takes stochastic gradients as input, is adaptive to a change of objective, and its parameters are independent of the required precision. In summary, we make the following contributions
\begin{itemize}
    \item We introduce a \textbf{single-loop}, practically efficient algorithm for optimization on the transport polytope.

     \item We provide theoretical guarantees for the \textbf{performance} under various assumptions, including OT (linear).

    \item We show that this algorithm can be adapted to handle different scenarios such as stochastic gradients, online settings, and related tensor problems.
\end{itemize}

\paragraph{Notations.} The standard Euclidean product for vectors and matrices is denoted by $\langle\cdot,\cdot\rangle$. For a positive integer $n$, we denote by $1_n = (1,\dots, 1)$ the vector of $\R^n$ with all ones, and by $[n]$ the set of integers from $1$ to $n$, included. We denote by $\Delta_n$ the probability simplex, defined as
\begin{equation*}
    \Delta_n = \big\{\nu \in\R^{n}: \nu \ge 0, \ \langle 1_n, \nu \rangle = 1 \big\}\, .
\end{equation*}
The analogue for probability matrices in $\R^{m \times n}$ (resp. tensors in $\R^{m_1 \times \ldots \times m_d}$) is defined {\em mutatis mutandis} and denoted by $\Delta_{m \times n}$ (resp. $\Delta_{m_1 \times \ldots \times m_d}$). For two reals $a, b$ we denote by $a \wedge b$ the minimum of $a$ and $b$. We extend this notation to vectors, meaning the entrywise minimum. The operator $\Diag(\cdot)$ yields a diagonal matrix in $\R^{m\times m}$ from a vector in $\R^m$, with diagonal entries $(\Diag(x))_{ii} = x_i$. We note the entrywise product for matrices $(A\odot B)_{ij} = A_{ij}B_{ij}$ and the entrywise division $(A\oslash B)_{ij} = A_{ij} / B_{ij}$. The notations  $\exp(A)$ and $\log(A)$ are used for the entrywise exponential and natural logarithm functions, i.e. $(\exp(A))_{ij} = e^{A_{ij}}$,  $(\log(A))_{ij} = \log({A_{ij}})$. The transpose of a matrix $A$ is noted $A^\top$.
We denote by $\Vert\cdot\Vert_\infty$ and $\Vert\cdot\Vert_1$ respectively the entrywise $\ell_1$ and $\ell_\infty$ norms on vectors and matrices.

\section{Problem and methods}
For two positive integers $m, n\geq 1$, let $\mu \in\R^m$ and $\nu \in\R^n$ be two probability distributions on $[m]$ and $[n]$ respectively, represented as vectors, elements of the simplexes $\Delta_m$ and $\Delta_n$. The transport polytope between $\mu$ and $\nu$, denoted by $\mathcal{T}(\mu, \nu)$, is a subset of $\Delta_{m \times n}$, set of $m\times n$ probability matrices. It contains all probability matrices $\gamma$ whose rows sum to $\mu$ and columns sum to $\nu$. It can be interpreted as the set of couplings, i.e. joint distributions between two variables with fixed marginals $\mu$ and $\nu$, defined as
\begin{equation}
    \mathcal{T}(\mu, \nu) = \left\{\gamma\in\R_+^{m\times n}:\ \gamma 1_n = \mu,\ \gamma^\top1_m = \nu\right\}.
\end{equation}
It is the intersection of the simplex $\Delta_{m\times n}$ with affine spaces given by $\cR(\mu) = \{\gamma \in \R^{m \times n}: \gamma 1_n=\mu\}$ and $\cC(\nu) = \{\gamma \in \R^{m \times n}:\gamma^\top 1_m=\nu\}$.
We consider in this work the problem of convex optimisation on the transport polytope:
\begin{equation}\label{opt_polytope}
    \min_{\gamma\in \mathcal{T}(\mu, \nu)} f(\gamma),
\end{equation}
where $f :\Delta_{m \times n} \to \R$ is a real-valued differentiable convex function that is defined on the set of probability matrices. The optimal transport problem is the particular case of \eqref{opt_polytope} where $f$ is linear:
\begin{equation}\label{kantorovitch}
        \min_{\gamma\in \mathcal{T}(\mu, \nu)} \langle C,\gamma\rangle.
\end{equation}

\subsection{Mirror descent and Sinkhorn.} A method that can be used to solve constrained convex optimization problems such as \eqref{opt_polytope} is Mirror Descent \cite{beck2003mirror}. For a convex differentiable function $\psi:\Delta_{m\times n}\to\R$, we define a Bregman divergence
\begin{equation*}
    D_\psi(x, y) = \psi(x) - \psi(y) - \langle\nabla\psi(y),x-y\rangle.
\end{equation*}
 The function $\psi$ plays the role of a barrier function, it is such that for every compact $K\subset \R$, its inverse image $\psi^{-1}(K)=\{x\in\Delta_{m\times n}:\ \psi(x)\in K\}$ is in the interior of the polytope of constraints. We recall that $\cT(\mu, \nu)$ is the intersection of the probability simplex $\Delta_{m\times n}$ with the affine spaces $\cR(\mu)$ and $\cC(\nu)$. A choice of function $\psi$ that guarantees that the iterates are in the interior of the simplex $\Delta_{m\times n}$ is the negative entropy, defined by
\begin{equation*}
    \forall \gamma\in\R^{m\times n}_+,\ \psi(\gamma) = -H(\gamma)=\langle \gamma, \log\gamma\rangle,
\end{equation*}
and the resulting Bregman divergence is the relative entropy
\begin{equation*}
    D_{\KL{}}(A, B) = \langle A, \log(A)-\log(B)\rangle + \langle B - A, 1_{m\times n}\rangle.
\end{equation*}
Under these conditions, a mirror descent algorithm for \eqref{opt_polytope} would define its iterates by
\begin{equation}
\label{eq:mirror-descent}
    \gamma_{t+1} = \argmin_{\mathcal{T}(\mu, \nu)} \left\{ \eta_t \langle \gamma, \nabla f(\gamma_t)\rangle + D_{\KL{}}(\gamma, \gamma_t)\right\}
\end{equation}
with $\eta_t$ the step-size at time $t$.
This update requires solving an entropic-regularised optimal transport problem at each iterate
\begin{equation*}
    \gamma_{t+1} = \argmin_{\mathcal{T}(\mu, \nu)} \left\{ \langle \gamma, C_t\rangle -\alpha_t H(\gamma)\right\}
\end{equation*}
with regularisation parameter $\alpha_t = 1/\eta_t$ and cost matrix $C_t = \nabla f(\gamma_t) - \log(\gamma_t)/\eta_t$. However, the oracle for this problem is not explicit. Solving this step approximately within a mirror descent loop has been used in Gromov-Wasserstein problems
\cite{peyre2016gromov}. State of the art algorithms to solve this intermediate problem involve accelerated gradient descent schemes \cite{dvurechensky2018computational, lin2019efficient}. A simple and popular algorithm to tackle this inner problem is the Sinkhorn matrix scaling algorithm \cite{sinkhorn1964relationship,cuturi2013sinkhorn}. It consists in projecting alternatively an initial matrix $\gamma_1 = e^{-C/\alpha}$ on the marginal spaces with Bregman projections: if $t$ is odd
\begin{equation*}
    \gamma_{t+1} = \argmin_{\gamma 1_n = \mu} \left\{D_{\KL{}}(\gamma, \gamma_t)\right\}
\end{equation*}
and if $t$ is even
\begin{equation*}
    \gamma_{t+1} = \argmin_{\gamma^\top 1_m = \nu} \left\{D_{\KL{}}(\gamma, \gamma_t)\right\}.
\end{equation*}
Each of these iterates only relies on proportionately scaling the rows and columns of the matrix. An algorithm using several steps of the Sinkhorn algorithm to approximate mirror descent uses \textbf{nested loops}.

\subsection{Mirror Sinkhorn.}
\label{sec:test}
The algorithm that we propose consists of alternating a step of entropic mirror descent and a step of Sinkhorn algorithm. Similarly to mirror descent, it can be written in a proximal form where the optimisation is performed on each marginal at a time: if $t$ is even, 
\begin{equation*}
    \gamma_{t+1} = \argmin_{\gamma 1_n = \mu} \left\{ \eta_t \langle \gamma, \nabla f(\gamma_t)\rangle + D_{\KL{}}(\gamma, \gamma_t)\right\}
\end{equation*}
if $t$ is odd,
\begin{equation*}
    \gamma_{t+1} = \argmin_{\gamma^\top 1_m = \nu} \left\{ \eta_t \langle \gamma, \nabla f(\gamma_t)\rangle + D_{\KL{}}(\gamma, \gamma_t)\right\}.
\end{equation*}
Formally, Mirror Sinkhorn algorithm is defined in the following manner (see Figure~\ref{fig:alg-explain} for an illustration).
\begin{algorithm}

\caption{Mirror Sinkhorn}
\label{alg:main}
\KwData{Initialise $\gamma_1 =\overline{\gamma}_1= \mu\nu^\top$, define stepsize $\eta_t$.}
\For{$1\leq t\leq T$}{
Update $\gamma_{t+1}' = \gamma_{t} \odot \exp\left(- \eta_t\nabla f(\gamma_t)\right),$\\
\eIf{$t$ is even}{
Rows: 
$\gamma_{t+1} = \Diag\left(\mu\oslash(\gamma_{t+1}'1_n)\right)\gamma_{t+1}'$,
}{
Cols: $\gamma_{t+1} = \gamma_{t+1}'\Diag\left(\nu\oslash((\gamma_{t+1}')^\top 1_m)\right)$,
}
Update $\overline{\gamma}_{t+1} = \frac{t}{t+1}\overline{\gamma}_t + \frac{1}{t+1}\gamma_{t+1}$
}
\KwOut{$\overline{\gamma}_T = \frac{1}{T}\sum_{t=1}^T\gamma_t$ }

\end{algorithm}
\begin{figure}[ht]
\vskip 0.2in
\begin{center}
\centerline{\includegraphics[width=0.8\columnwidth]{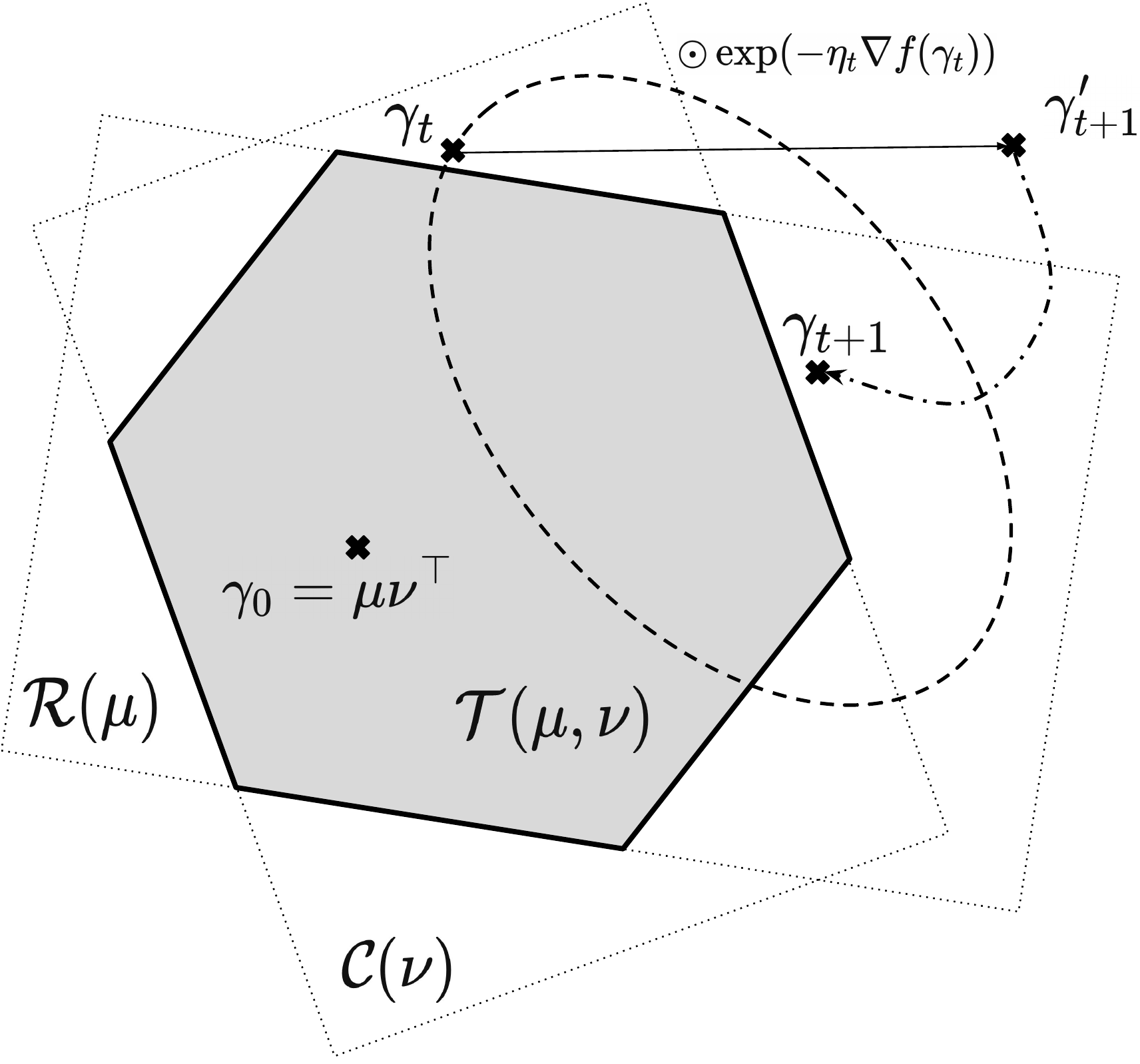}}
\caption{An illustration of the {\em Mirror Sinkhorn} algorithm: At each iteration, a gradient step on $\gamma_t$ yields the intermediate $\gamma'_{t+1}$, which is projected by row or column scaling onto either $\cR(\mu)$ or $\cC(\nu)$.}
\label{fig:alg-explain}
\end{center}
\vskip -0.2in
\end{figure}

Each step of the algorithm has a time complexity of $O(mn)$. Similarly to the Sinkhorn algorithm, this algorithm can be parallelized \cite{cuturi2013sinkhorn}. Importantly, this a \textbf{single-loop} algorithm, in \textbf{stark contrast} with nested loop algorithms that rely on iterative algorithms within a loop of gradient updates. Our algorithms handles the affine constraints given by $\cR(\mu)$ and $\cC(\nu)$ by a \textbf{single and explicit normalization}. This makes the algorithm particularly easy to implement and efficient. Theoretical results providing convergence rates under different assumptions are provided in Section~\ref{sec:th}. Some settings require minor variations (e.g. use of stochastic gradients, absence of running averaging). Even though we describe them in the corresponding sections, we also include them in the appendix for the sake of completeness.

\paragraph{Regularizing the marginals.} One approach to guarantee that the constraints are satisfied is to add a positive regularization term that is minimized when the constraints are satisfied. The following proposition states that Algorithm \ref{alg:main} has the same iterates whether it is based on the regularized objective or the unregularized objective. 
\begin{prop}
\label{prop:reg}
Let $f:\Delta_{m\times n}\to\R$ be a differentiable convex function, $R_\mu: \Delta_m\to \R_+$ be a differentiable convex function such that $\nabla R_\mu(\mu')=0$ if and only if $\mu' = \mu$, and $R_\nu: \Delta_n\to \R_+$ be defined analogously. We define the following regularized objective
\begin{equation*}
    f^{\text{reg}}: \gamma\mapsto f(\gamma)+R_\mu(\gamma 1_n)+R_\nu(\gamma^\top 1_m).
\end{equation*}
Let $(\gamma_t)_{t\geq 1}$ be the iterates of Mirror Sinkhorn with objective $f$ and let $(\gamma_t^{\text{reg}})_{t\geq 1}$ be the iterates with objective $f^{\text{reg}}$. Then
\begin{equation*}
    \forall t\geq 1,\ \gamma_t^{\text{reg}} = \gamma_t.
\end{equation*}
\end{prop}

\paragraph{Rounding for constraints satisfaction} The output of our algorithm does not necessarily belong to $\mathcal{T}(\mu, \nu)$. This can easily be remedied by applying an elegant rounding algorithm by \citet{altschuler2017near} to obtain a nearby matrix that belongs to the transport polytope.
\begin{algorithm}[hbt!]
\caption{Rounding algorithm \cite{altschuler2017near}}\label{alg:round}
\KwData{Matrix $\gamma$, target marginals $\mu, \nu$.}
Normalise rows $\gamma' = \Diag\left(\frac{\mu}{\gamma1_n}\wedge 1_m\right)\gamma$,

Normalise columns $\gamma'' = \gamma'\Diag\left(\frac{\nu}{(\gamma')^\top 1_m}\wedge 1_n\right)$,

\KwOut{$\ROUND(\gamma) = \gamma'' + \left(\frac{\mu - \gamma''1_n}{\Vert \mu - \gamma''1_n \Vert_1}\right)(\nu - (\gamma'')^T1_m)^\top$ }
\end{algorithm}

One of the appeals of this function is that the $\ell^1$ distance in $\R^{m \times n}$ between its input and its output can be controlled by the $\ell^1$ distance between their marginals in $\R^m$ and $\R^n$ 
 \begin{equation*}
    \Vert\ROUND(\gamma) - \gamma\Vert_1 \leq 2\left(\Vert\gamma 1_n-\mu\Vert_1+\Vert\gamma^\top1_m-\nu\Vert_1\right)\, ,
\end{equation*}
and that its output is in the transport polytope
\begin{equation*}
    \forall \gamma\in\R^{m\times n}_+,\ \ROUND(\gamma)\in\mathcal{T}(\mu, \nu).
\end{equation*}
This algorithm has complexity $O(mn)$, which implies that it does not add asymptotic complexity if applied at the end of Algorithm \ref{alg:main}. It is also parralelization-friendly. 

To keep track of the constraint violation in the theoretical guarantees, we define $c(\cdot): \Delta^{m \times n} \to \R$ by
\begin{equation}
\label{eq:constraint}
    c(\gamma) = \Vert\gamma1_n-\mu\Vert_1 + \Vert\gamma^\top1_m-\nu\Vert_1.
\end{equation}
We also define for $\mu \in \Delta_m, \nu \in \Delta_n$ the constant $\delta$ appearing in our theoretical guarantees
\begin{equation}
    \label{eq:radius}
\radius = \Vert \log \mu\Vert_\infty + \Vert \log \nu\Vert_\infty\, .
\end{equation}

\paragraph{Use for optimal transport problem} As described above, this algorithm can be used to tackle the OT problem. In order to provide an $\varepsilon$-close optimal transport plan, the Sinkhorn algorithm is initialized using the matrix $(e^{-C_{i,j}/\alpha})_{i,j}$ with $\alpha/\varepsilon$ being a constant which depends on the choice of $\mu$ and $\nu$ - see \cite{altschuler2017near}. The target error needs to be known at initialization. In contrast, the Mirror Sinkhorn algorithm's initialization and stepsizes do not depend on $\varepsilon$.

Another widely used method for improving the computation of optimal transport problems is to transform the problem into that of sliced Wasserstein distance, considering averages over lower-rank projections \cite{bonneel2015sliced, kolouri2019generalized, le2019tree, nadjahi2021fast, niles2022estimation}. There are some conceptual similarities with certain aspects of our method, e.g. when the use of lower-rank estimates of the gradient are used in a stochastic setting (see discussion in Section~\ref{sec:stoch_setting}).
\section{Theoretical guarantees}\label{sec:th}
We present theoretical guarantees on the performance of the Mirror Sinkhorn algorithm in several settings. The proposed algorithm is evaluated under a variety of conditions, with small adaptations in each case. Proofs are in the Appendix.

\subsection{Online setting}

The Mirror Sinkhorn algorithm can be used in an online setting, which is most general and presented first. In it, there is not a unique function $f$ but a stream of functions $f_t$ and the performance of the algorithm is evaluated as a regret bound \citep[see, e.g.][]{bubeck2012regret}. The gradient can be replaced by $\nabla f_t(\gamma_t)$, to run the algorithm in this setting. The bounds on the worst-case regret shown in the following illustrate the claim that  Mirror Sinkhorn is adaptive. 

\begin{algorithm}
\caption{Online Mirror Sinkhorn}\label{alg:online}
\KwData{Initialise $\gamma_1 =\overline{\gamma}_1= \mu\nu^\top$, define stepsize $\eta_t$, stream of loss functions $(f_t)$.}
\For{$1\leq t\leq T$}{
Update $\gamma_{t+1}' = \gamma_{t} \odot \exp\left(- \eta_t\nabla f_t(\gamma_t)\right),$\\
\eIf{$t$ is even}{
Rows: $\gamma_{t+1} = \Diag\left(\mu\oslash(\gamma_{t+1}'1_n)\right)\gamma_{t+1}'$,
}{
Cols: $\gamma_{t+1} = \gamma_{t+1}'\Diag\left(\nu\oslash((\gamma_{t+1}')^\top 1_m)\right)$,
}
}
\KwOut{$\gamma_T$ }

\end{algorithm}

\begin{thm}
\label{thm:online}
For $(f_t)_{t\geq 1}$ a sequence of convex functions that are $B$-Lipschitz for the norm $\Vert\cdot\Vert_1$, let
$\radius$ be as defined in Equation~(\ref{eq:radius}), $\eta_t = \frac{1}{B}\sqrt{\frac{\radius}{t}}$. Then, the iterates $(\gamma_t)_{t\geq 1}$ of Mirror Sinkhorn as described in Algorithm~\ref{alg:online} satisfy the following regret bound:
\begin{equation*}
\max_{\gamma\in\mathcal{T}(\mu, \nu)}
    \sum_{t=1}^T\left(f_t(\gamma_t) - f_t(\gamma) \right)\leq  \frac{9B}{8}\sqrt{\radius T}(2+\log(T)),
\end{equation*}
and the constraints satisfy
\begin{equation*}
    \sum_{t=1}^Tc(\gamma_t) \leq  \frac{3}{2}\sqrt{\radius T}(2+\log(T)),
\end{equation*}
with $c(\cdot)$ as defined in Equation (\ref{eq:constraint}).

Rounding at each $t$, $\tilde \gamma_t = \ROUND(\gamma_t)$ with Alg.~\ref{alg:round}, it holds that
\begin{equation*}
\max_{\gamma\in\mathcal{T}(\mu, \nu)} 
    \sum_{t=1}^T\left(f_t(\tilde \gamma_t) - f_t(\gamma)\right) \leq  \frac{9B}{8}\sqrt{\radius T}(2+\log(T)).
\end{equation*}
\end{thm}

\subsection{Stochastic setting}
\label{sec:stoch_setting}

The Mirror Sinkhorn algorithm can be used both with deterministic and stochastic gradient updates on the function $f$. In the first case, as described in Algorithm~\ref{alg:main}, the updates are given by $\nabla f(\gamma_t)$, and in the stochastic case by $g_t$ satisfying $\Esp{}{g_t\vert \gamma_t} = \nabla f(\gamma_t)$. This minor adaptation is fully described in Algorithm~\ref{alg:stoch} in the appendix for completeness.

This setting is common in stochastic optimization and allows this algorithm to be used in learning tasks where the function $f$ is the expectation of a data-dependent loss, and the $g_t$ are gradients of this loss for one data observation (or a mini-batch thereof). This also illustrates that the algorithm is not sensitive to noise in the gradients. The complexity of each step of the stochastic algorithm can be further reduced by subsampling the gradient. To summarize, this setting can be motivated in several situations:
\begin{itemize}
    \item Gradient subsampling on a random set of indices $S$, 
    \begin{equation*}
        g_t = \frac{mn}{\vert S\vert}\sum_{(i, j)\in S}(\nabla f(\gamma_t))_{ij} e_{ij}\, .
    \end{equation*}
    \item Gradient observation with random additive noise 
    \begin{equation*}
        g_t = \nabla f(\gamma_t) + \sigma Z_t\, .
    \end{equation*} 
    such that $\Esp{}{Z_t}=0$ and $\Esp{}{\Vert Z_t\Vert_\infty^2}<\infty$.
    \item For OT problems, the objective function is $f:\gamma\mapsto\langle \gamma, C\rangle$, and its gradient $\nabla f(\gamma)$ given by the cost matrix $C$ for all $\gamma$. In the 2-Wasserstein distance (Euclidean cost), it is equivalent to having $C_{ij} = - \langle x_i, y_j\rangle$. If we observe $(X_i)$, $(Y_j)$ two independent families of  random vectors in $\R^d$ such that $\Esp{}{X_i}=x_i\in \R^d$, $\Esp{}{Y_j}=y_j\in \R^d$, we can use
    \begin{equation*}
        (g_t)_{ij} = - \langle X_i, Y_j\rangle\, ,
    \end{equation*} 
    as an example of $g_t = C_t$ of stochastic observation of the cost in an optimal transport problem.

\end{itemize}

All the results presented here can be directly applied to the deterministic case by setting $\sigma = 0$. All convergence bounds are anytime, meaning that the number of iterations $T$ is not known when choosing the step-size. As a consequence, there is an additional logarithmic term in the rate of convergence. This term can be removed if the stepsize is chosen to be constant in $t$, but dependent on $T$. 
\begin{thm}\label{thm:stoch}
For $f$ convex and $B$-Lipschitz for the norm $\Vert\cdot\Vert_1$, and
\begin{equation*}
    \Esp{}{\Vert g_t - \nabla f(\gamma_t)\Vert_\infty^2\vert \gamma_t}\leq \sigma^2,
\end{equation*}
let $B_\sigma = \sqrt{B^2 +\sigma^2}$, $\eta_t = \frac{1}{B_\sigma}\sqrt{\frac{\radius}{t}}$. Then, taking $\overline{\gamma}_T$ to be the output  after $T$ steps of the Mirror Sinkhorn as described in Algorithm~\ref{alg:stoch} and $\tilde \gamma_T = \ROUND(\overline{\gamma}_T)$, it holds that
\begin{equation*}
    \Esp{}{f(\tilde \gamma_T) - f(\gamma^*)} \leq  \frac{B_\sigma}{2}\sqrt{\frac{\radius}{T}}(1+\log(T)).
\end{equation*}
With constant stepsize $\eta_t = \frac{1}{B_\sigma}\sqrt{\frac{2\radius}{T}}$ we have
\begin{equation*}
    \Esp{}{f(\tilde \gamma_T) - f(\gamma^*)} \leq  B_\sigma\sqrt{\frac{2\radius}{T}}.
\end{equation*}
\end{thm}

\subsection{Optimal Transport}
The Mirror Sinkhorn algorithm can be applied to the Optimal Transport problem (OT) described in \eqref{kantorovitch}, when $f$ is a linear form. The Sinkhorn algorithm is widely used to tackle this problem, but suffers from some limitations in its most common form. First, it requires to have access to the exact cost matrix at the start of the algorithm. In contrast, as shown above, we can apply Mirror Sinkhorn to a stream of random cost matrices (stochastic gradients of the linear problem - see Algorithm \ref{alg:stoch}, \ref{alg:ot} for details). Further, the Sinkhorn algorithm solves an entropic-regularized version of OT \cite{cuturi2013sinkhorn} and as such suffers from a regularization bias for any fixed $\alpha>0$ - which must be set small enough, as a function of the desired precision $\varepsilon >0$ for OT \cite{altschuler2017near}. As noted above, this modification of the problem can be considered ``a feature rather than a bug'', as for any fixed $\alpha > 0$, the new solution $\gamma_\alpha^*$ has enviable properties, such as differentiability. Such regularized objectives can also be tackled by Mirror Sinkhorn, as described in Section~\ref{sec:strong-cvx}


Applying Mirror Sinkhorn either with $\nabla f(\gamma) = C$ (deterministic case) or $g_t = C_t$ (stochastic case, as in Algorithm~\ref{alg:stoch}), with successive random cost matrices $C_t$, allows to solve OT, with theoretical guarantees given in the following result, a corollary of Theorem~\ref{thm:stoch} (see Algorithm~\ref{alg:ot} in the appendix for full details).

\begin{thm}\label{thm:ot_round}
Let $(C_t)_{t\geq 1}$ be a sequence of random cost matrices such that $\Esp{}{C_t} = C$, for some cost matrix \mbox{$C\in\R^{m\times n}$} satisfying $\Vert C\Vert_\infty\leq 1$, and $\Esp{}{\Vert C_t - C\Vert_\infty^2}\leq \sigma^2$.
Setting  $\eta_t = \sqrt{\frac{\radius}{(1+\sigma^2)t}}$ in Mirror Sinkhorn as described in Algorithm~\ref{alg:ot}, for the output $\overline{\gamma}_T$ after $T$ steps and $\tilde \gamma_T = \ROUND(\overline{\gamma}_T)$, it holds that
\begin{equation*}
    \Esp{}{\langle C, \tilde \gamma_T\rangle} - \langle C, \gamma^*\rangle \leq  \frac{9}{8}\sqrt{\frac{(1+\sigma^2)\radius}{T}}(2+\log(T)).
\end{equation*}
\end{thm}
\paragraph{Constant stepsize.} We also provide the result for constant stepsize, which is asymptotically the same as the rate for Sink/Greenkhorn \cite{lin2019efficient}, with optimized constant.
\begin{thm}\label{thm:ot_round_fixedstep}
Let $(C_t)_{t\geq 1}$ be a sequence of random cost matrices with $\Esp{}{C_t} = C$, for $C\in\R^{m\times n}$ that satsifies $\Vert C\Vert_\infty\leq 1$, and $\Esp{}{\Vert C_t - C\Vert_\infty^2}\leq \sigma^2$.
For $\varepsilon>0$, and $\eta_t = \varepsilon\sqrt{\frac{\radius}{1+\sigma^2}}$. Then, taking the output $\overline{\gamma}_T$ after $T$ steps of Mirror Sinkhorn as described in Algorithm \ref{alg:ot} and $\tilde \gamma_T = \ROUND(\overline\gamma_T)$, it holds for $T\geq 5(1+\sigma^2)\radius\varepsilon^{-2}$ that
\begin{equation*}
    \Esp{}{\langle C, \ROUND(\tilde\gamma_T)\rangle} - \langle C, \gamma^*\rangle\leq \varepsilon\, .
\end{equation*}

\end{thm}

\subsection{Strong Convexity and Smoothness}
\label{sec:strong-cvx}

We consider here the specific case where $f$ is both $\ell$-strongly convex and $L$-smooth w. r. t. the relative entropy, i.e.
\begin{equation*}\label{eqconvex}
    f(\gamma')-f(\gamma) - \langle\nabla f(\gamma), \gamma'-\gamma\rangle \geq \ell D_{\KL} (\gamma', \gamma)
\end{equation*}
and
\begin{equation*}\label{eqsmooth}
    f(\gamma')-f(\gamma) - \langle\nabla f(\gamma), \gamma'-\gamma\rangle \leq LD_{\KL} (\gamma', \gamma)\, .
\end{equation*}
Under such assumptions, the algorithm converges at a faster rate, which fits with the results on the convergence of the Sinkhorn algorithm for an entropy-regularized objective.


\begin{thm}\label{thm:strong}
For $f$ being $\ell$-smooth and $L$-strongly convex with respect to the relative entropy, let $\eta_t = \frac{1}{\ell t}$. Then, taking $\overline{\gamma}_T$ the output of Mirror Sinkhorn (Algorithm \ref{alg:main}) after $T$ steps , it holds for $B = \Vert \nabla f(\gamma^*)\Vert_\infty$ that
\begin{equation*}
    f(\overline\gamma_T)-f(\gamma^*) + 2Bc(\overline\gamma_T) \leq \frac{(2B+L)^2}{8\ell T}(1+\log(T)),
\end{equation*}
with $c(\gamma) = \Vert \gamma 1_m-\mu\Vert_1+\Vert \gamma^\top 1_n-\nu\Vert_1$.

\end{thm}
Note that by definition of the relative entropy as Bregman divergence of the negative entropy $\gamma\mapsto -H(\gamma)$, the latter is $1$-strongly convex and 1-smooth with respect to the former. Thus, any $f$ of the form
\begin{equation*}
    f(\gamma) = \langle C,\gamma\rangle - \alpha H(\gamma)
\end{equation*}
is $L$-smooth and $\ell$-strongly convex with $L = \ell = \alpha$. Theorem~\ref{thm:strong} therefore applies to entropic-regularized OT.

\subsection{Tensor Case}
The Mirror Sinkhorn algorithm can be applied to a generalization involving probability tensors with multiple marginal constraints. Let $m_1, \dots, m_d\geq 1,$ be $d$ positive integers. As noted in the definitions, the set of probability tensors (nonnegative tensors with entries summing to $1$) is denoted by $\Delta_{m_1\times\cdots\times m_d}$. For $\mu_1\in\Delta^{m_1}, \dots, \mu_d\in\Delta^{m_d}$ probability vectors, the multiple transport polytope with marginals $\mu_1, \dots, \mu_d$ is the set of  tensors defined by 
\begin{equation*}
    \mathcal{T}(\mu_1,\dots,\mu_d) = \big\{\gamma\in\R^{m_1\times\cdots\times m_d}_+:\forall k,\, S_k(\gamma) = \mu_k\big\}, 
\end{equation*}
where $S_k(\gamma)$ is the sum of  $\gamma$ across all dimensions but $k$:
\begin{equation*}
    \forall{1\leq i_k\leq m_k},\ (S_k(\gamma))_{i_k} = \sum_{j: j_k = i_k}\gamma_{j_1,\dots,j_d}.
\end{equation*}
For a convex function $f:\Delta_{m_1\times\cdots\times m_d}\to\R$, the optimization problem described in Equation \eqref{opt_polytope} generalizes to
\begin{equation*}
    \min_{\gamma\in \mathcal{T}(\mu_1,\dots,\mu_d)} f(\gamma).
\end{equation*}
The Mirror Sinkhorn algorithm can also be adapted to tackle these problems, choosing at each step the dimension along which there is the largest constraint variation to normalize (see Algorithm~\ref{alg:tensor} in the Appendix for a full description).

\begin{thm}\label{thm:tensor}
Let $f :\Delta_{m_1\times\ldots\times m_d} \to \R$ convex and $B$-Lipschitz with respect to the norm $\Vert\cdot\Vert_1$, and with \mbox{$\radius = \sum_{k=1}^d\Vert \log \mu_k\Vert_\infty$},  $\eta_t = \frac{1}{B}\sqrt{\frac{\radius}{t}}$. Taking $\overline{\gamma}_T$ the output after $T$ steps of Algorithm \ref{alg:tensor}, it holds that
\begin{equation*}
    f(\overline{\gamma}_T)-f(\gamma^*) \leq \frac{9B}{8}\sqrt{\frac{\radius}{T}}\left(2+\log(T)\right)\, ,
\end{equation*}
and the constraints violations are bounded as follows
\begin{equation*}
    \sum_{k=1}^d\Vert S_k(\overline{\gamma}_T) - \mu_k\Vert_1\leq \frac{3}{2}\sqrt{\frac{\radius}{T}}\left(2+\log(T)\right).
\end{equation*}
\end{thm}

\begin{figure*}[!ht]
    \centering
    \includegraphics[width=\textwidth]{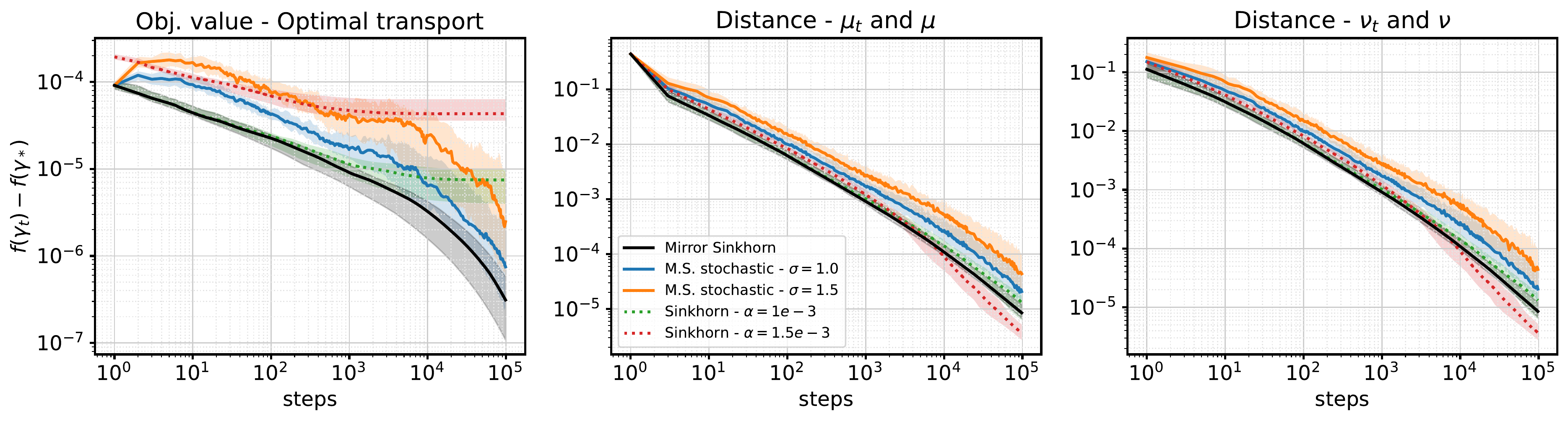}
        \vspace{-0.7cm}
    \caption{Performance of the Mirror Sinkhorn algorithm on the optimal transport problem, i.e. when $f$ is linear. Shown are exact gradient oracles (black) and stochastic (blue, orange). The experiment is described in Section~\ref{sec:ot-expe}. Our proposed algorithm is compared to the Sinkhorn algorithm (red, green, dotted lines) with several values of regularization parameter.
    \textbf{Left:} The value of $f(\gamma_t)$ as a function of $t$ - Note $f(\gamma^*) = 0$. \textbf{Center / Right:} The $\ell_1$ distance between the two marginals $(\mu_t, \nu_t)$ of $\gamma_t$ and constraints $(\mu, \nu)$ at any step.
    }
    \label{fig:ot-graphs}
\end{figure*}

\begin{figure*}[!ht]
    \centering
    \includegraphics[width=\textwidth]{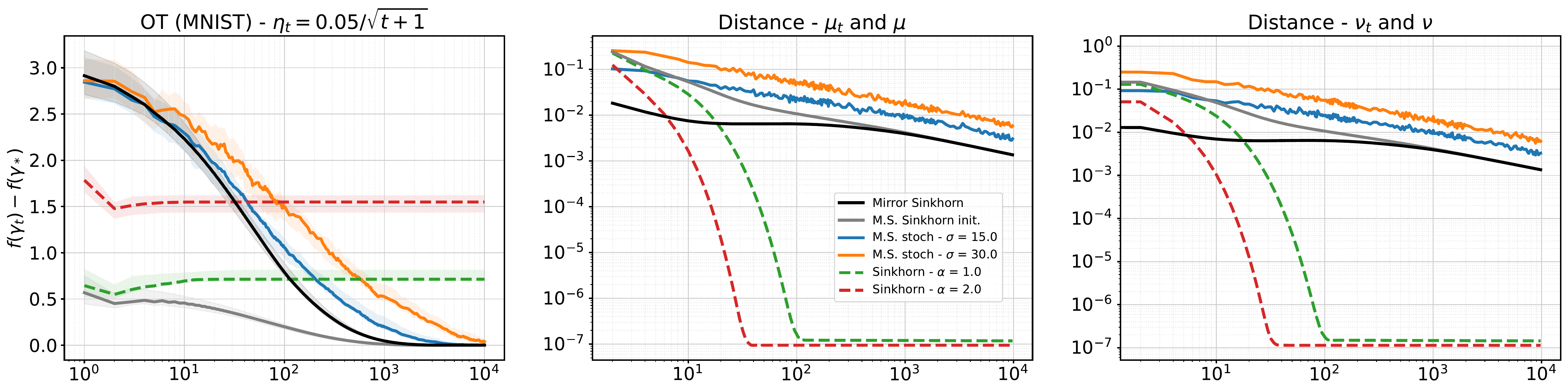}
        \vspace{-0.7cm}
    \caption{Performance of Mirror Sinkhorn on OT for MNIST data \citep{altschuler2017near}. See Figure~\ref{fig:ot-graphs} and Section~\ref{sec:ot-expe} for details.}
    \label{fig:ot-real-graphs}
\end{figure*}
\section{Experiments}
\subsection{Optimal transport}
\label{sec:ot-expe}

We present the performance of Mirror Sinkhorn on a linear objective, i.e. the optimal transport problem
\[
\min_{\gamma \in \mathcal{T}(\mu, \nu)} \langle C, \gamma \rangle\, .
\]
We take $m,n=100$, random $C \in \R^{m \times n}$ with independent off-diagonal coefficients in $[0,1]$ and zero diagonal, $\mu = \nu$ random, so that the optimizer and the optimal value of the problem is known, and equal to $0$ (with a different $C$ for each run).

We compare the performance of our algorithm to that of Sinkhorn's algorithm for different values of regularization parameter $\alpha>0$, running for $T = 10^5$ steps. The number of steps is chosen very large to illustrate the convergence rate, but this is not required to obtain a low optimization error, as shown in these experiments. We run this experiment both with $C_t = \nabla f(\gamma_t) = C$ (where the gradient is exact), and $C_t = C + \sigma Z_t$, where $Z_t$ has i.i.d. coefficients (stochastic gradients). These are reported in Figure~\ref{fig:ot-graphs}, for $32$ runs, with median and 10th-90th percentile. Our conclusion is that Mirror Sinkhorn is a fast and efficient algorithm to solve the optimal transport problem. In particular, it does not suffer from a regularization bias, and $\langle C, \gamma_t\rangle$ converges to its optimal value. There are two main advantages compared to using Sinkhorn for the optimal transport are: First, it is adaptive, there is no need to have a fixed desired precision, and to derive an corresponding regularization parameter. Second, it is an online algorithm, that can run on a stream of stochastic observations of $C$. This is not possible for some of the algorithms to either solve OT or its entropic regularized counterpart, including the Sinkhorn algorithm.

A consequence of the convergence of $\gamma_t$ to $\gamma^*$ is the slow progress for the two marginal constraints: as shown in our results, the violation in the constraint is polynomially decreasing in $t$, rather than the linear convergence (i.e. exponentially decreasing in $t$) of the Sinkhorn algorithm \cite{birkhoff1957extensions, carlier2022linear}. This phenomenon is particularly visible when the entropic regularization parameter $\alpha>0$ is higher, yielding solutions that are further from the boundary. The linear convergence is driven by a constant quantifying the distance of $\gamma^*_\alpha$ from the boundary of $\mathcal{T}(\mu, \nu)$, which therefore increases with $\alpha$. This is visible in Figure~\ref{fig:ot-graphs} and~\ref{fig:ot-real-graphs}, for the different values of $\alpha$, highlighting the inherent trade-off between regularization bias and constraint violation.

We also include an illustration of our method on two datasets used in \citet{altschuler2017near}, following their experimental setup: we use as instances of OT random pairs from MNIST (10 in total), and simulated SQUARES data consisting of pairs of $20 \times 20$ images with a light square of random size on a dark background (also 10 in total). We report the results of this experiment, and the comparison of Sinkhorn with our algorithm in various setting, in Figure~\ref{fig:ot-real-graphs} (for MNIST) and Figure~\ref{fig:ot-app-graphs} (in Appendix for SQUARES), we also include for ease of comparison, Mirror Sinkhorn with the initialization of the Sinkhorn algorithm (in gray). We note that the well-understood slower convergence of our algorithm over the marginals is of a much smaller order than the gain in function objective.

\begin{figure*}[!ht]
    \centering
    \includegraphics[width=\textwidth]{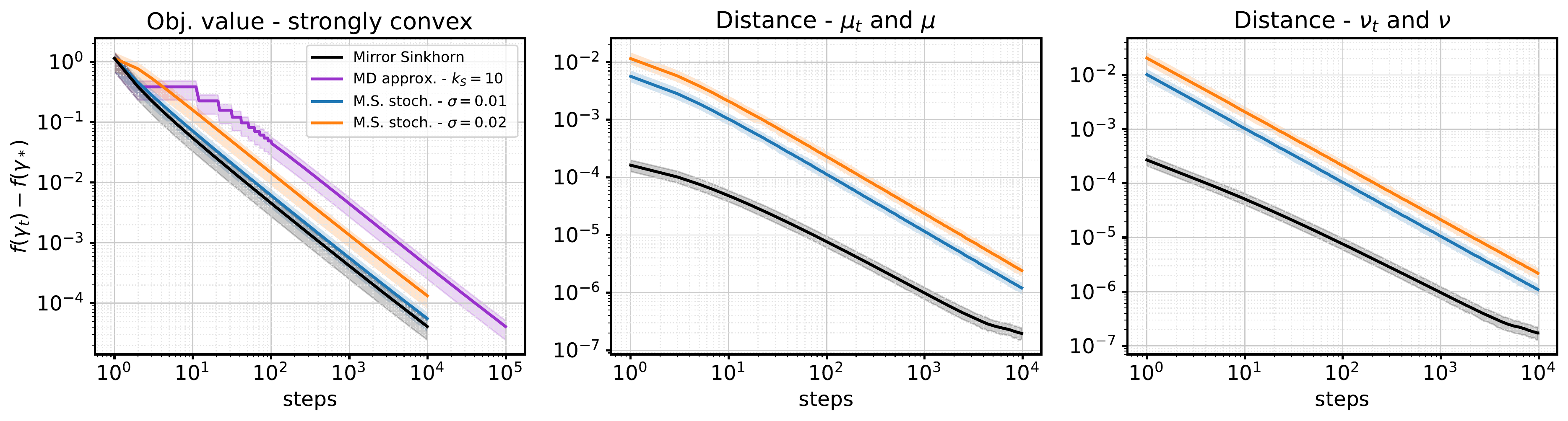}
    \vspace{-0.5cm}
    \caption{Performance of the Mirror Sinkhorn algorithm on strongly convex objectives, with exact gradient oracles (black) and stochastic gradients (blue and orange). The experiment is described in Section~\ref{sec:strong-expe}. This is compared with an approximation of mirror descent (purple) with $k_S = 10$ steps of Sinkhorn projection at each gradient update. \textbf{Left:} The value of $f(\gamma_t) - f(\gamma^*)$ as a function of $t$. \textbf{Center / Right:} The $\ell_1$ distance between the marginals $(\mu_t, \nu_t)$ of $\gamma_t$ and constraints $(\mu, \nu)$ before each normalization.}
    \label{fig:sc-graphs}
\end{figure*}

\subsection{Strongly convex optimization}
\label{sec:strong-expe}

We consider minimization of a smooth and strongly convex objective $f$ over $\cT(\mu, \nu)$ for randomly chosen $\mu, \nu$ of respective sizes $m, n$. We present an experiment where $f$ is a sum of several strongly convex objectives minimized at a common $\gamma^* \in \cT(\mu, \nu)$ randomly chosen in its interior (with a different $\gamma^*$ for each run). This allows to plot $f(\gamma_t) - f(\gamma^*)$ since the latter term is known. We take $m=50$, $n=60$ on this illustrative example, and run our algorithm on $T=10^4$ gradient update steps ($T$ is very large only for illustration purposes), with a stepsize regime proportional to $1/(t+1)$. Our algorithm is evaluated both with exact and stochastic gradient updates. In the latter case, the stochastic gradients are derived from the exact gradient by adding independent noise, allowing us to measure the impact of gradient noise.

The results are represented in Figure~\ref{fig:sc-graphs} for 1,024 independent runs, with median and 10th-90th percentile. They empirically confirm the speed of convergence of Theorem~\ref{thm:strong}. Our method is compared to an approximation of mirror descent, using a nested loop for the proximal step. We use a modified Mirror Sinkhorn algorithm with $k_{S}$ steps of alternating row/column normalization at each gradient update. 

Here we take $k_S = 10$, and observe that surprisingly, the convergence is of similar order as that of Mirror Sinkhorn (which can be interpreted as having $k_S = 1$). Compared to the algorithm that we propose, using a nested-loop algorithm doing several steps of normalization to mimick mirror descent yields a significant slowdown, with a multiplicative factor on the number of algorithmic steps. Comparing the results of our algorithm with this approach by comparing instead at each gradient update, using multiple normalization steps did not yield significant improvement over our approach, even with higher value for $k_S$. This strenghtens experimentally the choice of using only one normalization step at each gradient update in Mirror Sinkhorn.

\subsection{Point cloud registration - single cell data}
\label{sec:expe-wp}

We consider the problem of point cloud registration, i.e. mapping unmatched data sources, each composed of $n$ data points $X$ and $Y$, potentially in different dimensions $d_X$ and $d_Y$, based on common geometry, related to Gromov-Wasserstein and similar problems \cite{memoli2011gromov, djuric2012convex, peyre2016gromov, solomon2016entropic}. In the case where $d_X = d_Y$, this problem can be formulated as {\em Wasserstein-Procrustes} \cite{grave2019unsupervised}, where a convex relaxation is used for the problem of assigning the $n$ points of $X$ to those of $Y$, by finding a permutation of rows and columns that matches two Gram matrices $K_X$ and $K_Y$ of $X$ and $Y$. This relaxation can be generalized to any setting where two such matrices (similarity, or distances) can be formulated. It is written as
\[
\min_{\gamma \in \mathcal{T}(\mu, \nu)} \|K_X \gamma  - \gamma K_Y\|_2^2\, ,
\]
for $\mu, \nu$ uniform on $[n]$. This quadratic problem can be regularized by adding a $-\lambda \|\gamma\|_2^2$ penalty. For $\lambda > 0$ small enough, the objective function remains convex on $\mathcal{T}_{\mu, \nu}$. It is in general possible to consider this problem for any two matrices $K_X$ and $K_Y$ that represent pairwise correlations or distances between two sets of points, even if they are in different spaces, as in the Gromov-Wasserstein problem. We apply our algorithm to this functional, on single-cell measurement data. In this experiment, $X$ represents genetic expression and $Y$ chromatin accessibility. The SNARE-seq data \cite{chen2019high} consists of $1,047$ vectors in dimension $10$ and $19$ respectively. The ground-truth matching is known, allowing for interpretability of the result. Following the process in \cite{demetci2020gromov}, a $k$-NN connectivity graph is constructed from the correlations between data points. The matrix of pairwise shortest path distances on this graph between each pair of nodes is then used for $K_X$ and $K_Y$. Normalization and maximal capping is applied.

We minimize this functional by taking $\lambda = 3$, with a \mbox{$k$-NN} graph taken for $k = 5$. We recall that in this case, $n=1,047$. We are applying a step-size regime proportional to $1/(t+1)$, for $T=10^5$ steps. We analyze the relevance of this optimization problem to the motivating task of identifying the matching: we predict a binary matrix by thresholding $\gamma_T$ at $c/n$ for a small constant $c$, and compare it to the ground truth permutation matrix. This allows us to count the number of predicted, and of correct matches by comparing to the known correspondence. The results for the objective value and the corresponding prediction accuracy are reported in Figure~\ref{fig:wp-graphs}, and the distance between the marginals presented in Figure~\ref{fig:wp-graphs-app} in Appendix~\ref{app:exp}.

\begin{figure}[ht!]

\begin{center}
\centerline{\includegraphics[width=0.9\columnwidth]{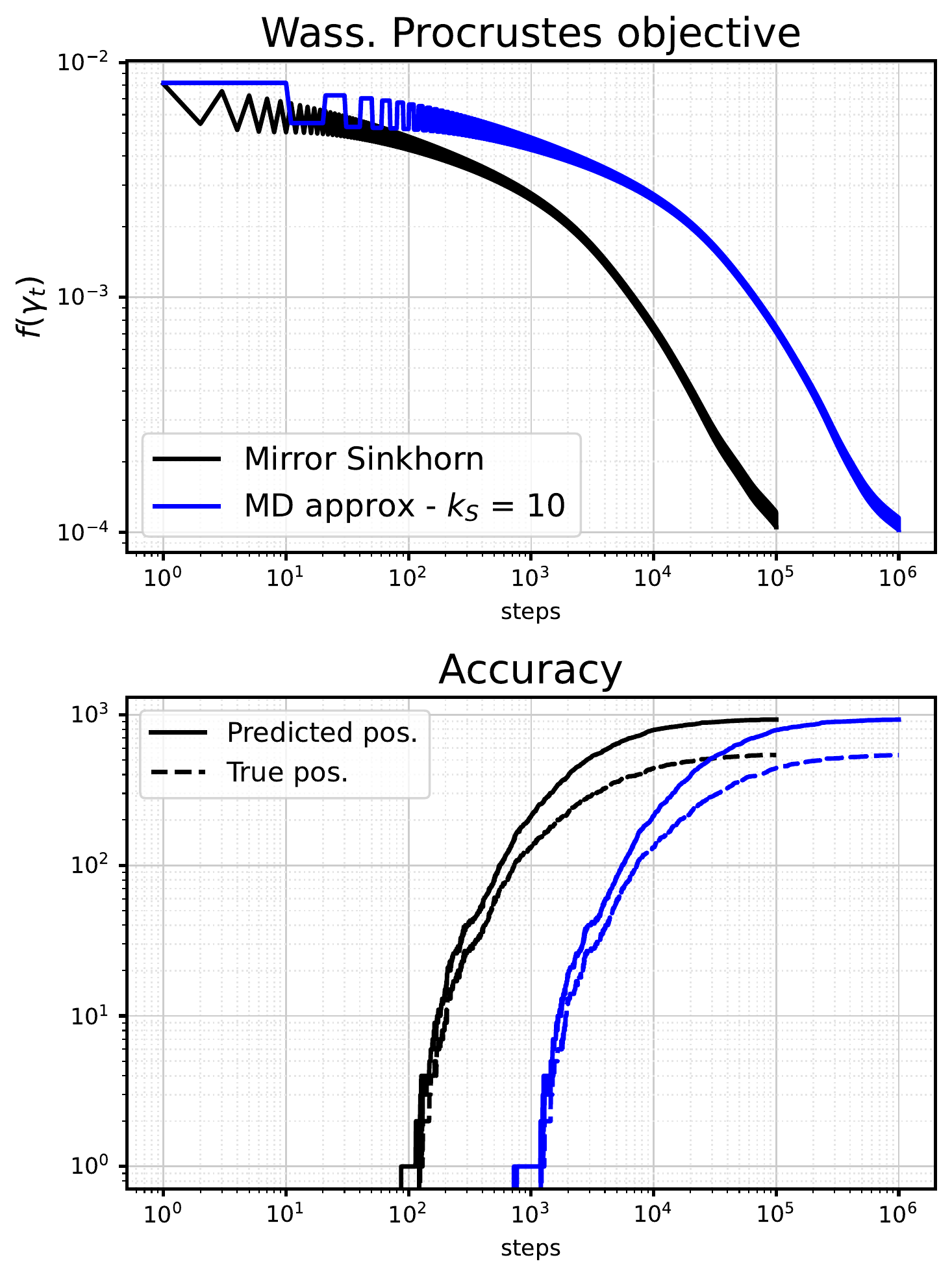}}
    \caption{Performance of the Mirror Sinkhorn algorithm (black) on the Wasserstein Procrustes objective, for the experiment described in Section~\ref{sec:expe-wp}. This is compared with an approximation of mirror descent (blue) with $k_S = 10$ steps of Sinkhorn projection at each gradient update. \textbf{Top:} The value of $f(\gamma_t)$ as a function of $t$. \textbf{Bottom}: For the predicted assignment matrix, thresholded version of $\gamma_t$, the number (compared to $n=1,047$) of predicted positives (solid line) and of true positives (dashed line) for both algorithms.}
    \label{fig:wp-graphs}
\end{center}
\vskip -0.2in
\end{figure}

\section{Concluding remarks}

We introduce an algorithm for convex optimization problems on transport polytopes, providing theoretical guarantees and empirical evidence for its performance on a wide range of problems, on simulated and real data.

Several questions are left open and could be interesting directions for future research. In the case of the optimal transport problem, there are natural connections between our approach and  annealing strategy for Sinkhorn \citep[see, e.g.][Remark 4.9]{peyre2019computational}. This is related to the question of taking several steps of Sinkhorn to approximate mirror descent in our comparisons, rather than $k_S = 1$, for nonlinear convex objectives. We did not find in our experimental results any significant improvement when taking $k_S>1$, but instead a significant computational overhead (due to the nested loops).

Other algorithmic approaches for OT and its entropic-regularized version have focused on optimization in the dual, as well as multimarginal, unbalanced, or partial versions. They have a different perspective (and no online aspect) and are focused on the linear objective \cite{dvurechensky2018computational, dvurechenskii2018decentralize, lin2019efficient, le2022multimarginal}. Our results, in particular Theorem~\ref{thm:ot_round_fixedstep}, can be seen in this larger context. As noted above, our analysis of the tensor setting is important, as it allows to treat the problem of multi-marginal optimal transport (MOT) - for linear objective functions $f$, and more generally for us, of any convex objective on this set. One of the main applications of the MOT setting is Wasserstein barycenters, and our work allows us to cover both these problems and generalizations \citep{pass2015multi, lin2022complexity, bigot2019data}.

Our results apply mostly on smooth (i.e. with Lipschitz gradients) convex functions, with additional results for strongly convex functions. Since results exist in the online learning literature on learning with Lipschitz losses \citep{hazan2007logarithmic}, it would also be a possible direction for research. Another interesting direction would be in nonconvex optimization, and the possiblity of approximating stationary points, on which there is an existing literature for mirror descent \cite{zhang2018convergence}. Finally, since our approach focuses on finite-dimensional iterates, semi-discrete approaches are not directly extendable to our setting. An extension in this direction, as considered by \citep{genevay2016stochastic}, would also be possible.

\section*{Acknowledgements} MB was supported by the Cantab Capital Institute for Mathematics of Information during this work. The authors would like to thank Jonathan Niles-Weed for discussions on this problem, and Laetitia~Meng-Papaxanthos for pointing them towards the SNARE-seq data. 

\newpage

\bibliography{references}
\bibliographystyle{icml2023}

\newpage
\appendix
\onecolumn
\section*{Appendix}
\appendix






\section{Main proofs}


\begin{proof}[Proof of Proposition~\ref{prop:reg}]
We first note
\begin{equation}
    (\nabla f^{\text{reg}}(\gamma))_{ij} = (\nabla f^{\text{reg}}(\gamma))_{ij} + (\nabla R_\mu(\gamma 1_n))_{i}+(\nabla R_\nu(\gamma^\top 1_m))_{j}.
\end{equation}
Let $t\geq 1$ be even, suppose that the iterate $\gamma_t^{\text{reg}}$ obtained when replacing $f$ by $f^{\text{reg}}$ in the Algorithm \ref{alg:main} is identical to $\gamma_t$. Then
\begin{align}
    (\gamma_{t+1}')^{\text{reg}}_{ij}&= (\gamma_{t+1}')_{ij}\exp\left(-\eta_t(\nabla R_\mu(\gamma_t1_n))_i -\eta_t (\nabla R_\nu(\gamma_t^\top 1_m))_j\right),\\
    &= (\gamma_{t+1}')_{ij}\exp\left(-\eta_t(\nabla R_\mu(\gamma_t1_n))_i\right),\\
\end{align}
since $\gamma_t^\top 1_m = \nu.$ The normalization of rows discards the gradient of the regularizer:
\begin{equation}
    (\gamma_{t+1}'1_m)^{\text{reg}}_{i} = (\gamma_{t+1}'1_m)_{i}\exp\left(-\eta_t(\nabla R_\mu(\gamma_t1_n))_i\right)
\end{equation}
so
\begin{align}
    (\gamma_{t+1})^{\text{reg}}_{ij} &= \frac{\mu_i}{(\gamma_{t+1}'1_m)^{\text{reg}}_{i}}(\gamma_{t+1}')^{\text{reg}}_{ij},\\
    &= \frac{\mu_i}{(\gamma_{t+1}'1_m)_{i}\exp\left(-\eta_t(\nabla R_\mu(\gamma_t1_n))_i\right)}(\gamma_{t+1}')_{ij}\exp\left(-\eta_t(\nabla R_\mu(\gamma_t1_n))_i\right),\\
    &= \frac{\mu_i}{(\gamma_{t+1}'1_m)_{i}}(\gamma_{t+1}')_{ij},\\
    &= (\gamma_{t+1})_{ij}.
\end{align}
The reasoning is the same for $t$ odd, and the initialization is true $\gamma_1 = \mu\nu^\top = \gamma_{1}^{\text{reg}}$. So $\gamma_t = \gamma_{t}^{\text{reg}}$ for all $t\geq 1$.
\end{proof}

\begin{thm}\label{thm:main_appendix}
If $f$ is $B$-Lipschitz for the norm $\Vert\cdot\Vert_1$, let $\radius = \Vert \log \mu\Vert_\infty + \Vert \log \nu\Vert_\infty$ and let the stepsize be $\eta_t = \frac{1}{B}\sqrt{\frac{\radius}{t}}$ in the algorithm \ref{alg:main}, then the output $\overline{\gamma}_T$ after $T$ steps verifies
\begin{equation}
    f(\overline{\gamma}_T) - f(\gamma^*) \leq  \frac{9B}{8}\sqrt{\frac{\radius}{T}}(2+\log(T)),
\end{equation}
and the constraints are close to be verified
\begin{equation}
    c(\overline{\gamma}_T) \leq  \frac{3}{2}\sqrt{\frac{\radius}{T}}(2+\log(T)),
\end{equation}
with
\begin{equation}
    c(\overline{\gamma}_T) = \Vert\overline{\gamma}_T1_n-\mu\Vert_1 + \Vert(\overline{\gamma}_T)^\top1_m-\nu\Vert_1.
\end{equation}
With constant step-size $\eta_t = \frac{1}{B}\sqrt{\frac{\radius}{T}}$ we have
\begin{equation}
    f(\overline{\gamma}_T) - f(\gamma^*) \leq  \frac{17B}{8}\sqrt{\frac{\radius}{T}}
\end{equation}
and
\begin{equation}
    c(\overline{\gamma}_T) \leq  2\sqrt{\frac{\radius}{T}}.
\end{equation}
\end{thm}

\begin{proof}
With Lemma \ref{lem:main} and Pinsker's inequality:
\begin{equation}
    D_{\KL}(\gamma^*, \gamma_t) - D_{\KL}(\gamma^*, \gamma_{t+1}) \geq 2\Vert\gamma_{t+1}- \gamma_t\Vert^2_1 + \eta_t\langle \nabla f(\gamma_t), \gamma_{t+1}-\gamma^*\rangle,
\end{equation}
We use the fact that $f$ is Lipschitz and convex:
\begin{equation}
    f(\gamma_t) - f(\gamma^*)\leq \langle \nabla f(\gamma_t), \gamma_{t+1}-\gamma^*\rangle + B\Vert \gamma_{t+1}- \gamma_t\Vert_1
\end{equation}
then
\begin{align}
    D_{\KL}(\gamma^*, \gamma_t) - D_{\KL}(\gamma^*, \gamma_{t+1}) \geq  \eta_t(f(\gamma_t) - f(\gamma^*)) + 2\Vert\gamma_{t+1}- \gamma_t\Vert^2_1 - \eta_t B\Vert \gamma_{t+1}- \gamma_t\Vert_1.
\end{align}
Given that $\gamma^* = \argmin_{\gamma\in\mathcal{T}(\mu, \nu)}f(\gamma)$, \begin{align}
    f(\gamma_t) - f(\gamma^*)&\geq f(\gamma_t) - f(\ROUND(\gamma_t))\\&\geq -B\Vert \gamma_t - \ROUND(\gamma_t)\Vert_1\\&\geq -2Bc(\gamma_t)\\
    &\geq -2B\Vert\gamma_{t+1}- \gamma_t\Vert_1
\end{align}
by Lemma \ref{lem:constraints} and Lemma \ref{lem:round}. So for $B'\geq 2B$
\begin{equation}
    0\leq \eta_t(f(\gamma_t) - f(\gamma^*) +B'c(\gamma_t))\leq  D_{\KL}(\gamma^*, \gamma_t) - D_{\KL}(\gamma^*, \gamma_{t+1}) +(B+B')\eta_t\Vert\gamma_{t+1}- \gamma_t\Vert_1 - 2 \Vert\gamma_{t+1}- \gamma_t\Vert_1^2.
\end{equation}

Moreover
\begin{equation}
    0\leq \eta_T(f(\gamma_t) - f(\gamma^*) +B'c(\gamma_t))\leq \eta_t(f(\gamma_t) - f(\gamma^*) +B'c(\gamma_t)),
\end{equation}
so
\begin{equation}\label{interm}
    0\leq \eta_T(f(\gamma_t) - f(\gamma^*) +B'c(\gamma_t))\leq  D_{\KL}(\gamma^*, \gamma_t) - D_{\KL}(\gamma^*, \gamma_{t+1}) +\frac{(B+B')^2}{8}\eta_t^2.
\end{equation}
Thus with $B' = 2B$,
\begin{equation}\label{beforesum}
    0\leq \eta_t(f(\gamma_t) - f(\gamma^*) +Bc(\gamma_t))\leq  D_{\KL}(\gamma^*, \gamma_t) - D_{\KL}(\gamma^*, \gamma_{t+1}) +\frac{9B^2\eta_t^2}{8}.
\end{equation}
We average over $1\leq t\leq T$,
\begin{equation}
    \frac{1}{T}\sum_{t=1}^T(f(\gamma_t) - f(\gamma^*) +Bc(\gamma_t)) \leq  \frac{1}{T\eta_T}D_{\KL}(\gamma^*, \gamma_1) + \frac{9B^2}{8}\frac{\sum_{k=1}^T\eta_t^2}{T\eta_T}.
\end{equation}
We remark $D_{\KL}(\gamma^*, \gamma_1)\leq\radius$. We use the convexity of $c$ and $f$ to extend the bound to the final iterate of the algorithm
\begin{equation}
    f(\overline{\gamma}_T) - f(\gamma^*) +2Bc(\overline{\gamma}_T) \leq  \frac{\radius}{T\eta_T} + \frac{9B^2}{8}\frac{\sum_{k=1}^T\eta_t^2}{T\eta_T}.
\end{equation}
We replace $\eta_t = \frac{1}{B}\sqrt{\frac{\radius}{t}}$
\begin{equation}\label{eq1}
    f(\overline{\gamma}_T) - f(\gamma^*) +2Bc(\overline{\gamma}_T) \leq  \frac{9B}{8}\sqrt{\frac{\radius}{T}}\left(\frac{8}{9}+1+ \log(T)\right).
\end{equation}
 We also infer from \eqref{interm}, with $B' = 5B$:
\begin{equation}
    3B\eta_T c(\gamma_t)\leq  D_{\KL}(\gamma^*, \gamma_t) - D_{\KL}(\gamma^*, \gamma_{t+1}) + \frac{9}{2}B^2\eta_t^2.
\end{equation}
We conclude by summing as well.
\end{proof}

\begin{thm}\label{thm:main_round_appendix}
If $f$ is $B$-Lipschitz for the norm $\Vert\cdot\Vert_1$, let $\radius = \Vert \log \mu\Vert_\infty + \Vert \log \nu\Vert_\infty$ and let the stepsize be $\eta_t = \frac{1}{B}\sqrt{\frac{\radius}{t}}$ in the algorithm \ref{alg:main}, then the output $\overline{\gamma}_T$ after $T$ steps verifies
\begin{equation}
    f(\ROUND(\overline{\gamma}_T)) - f(\gamma^*) \leq  \frac{9B}{8}\sqrt{\frac{\radius}{T}}(2+\log(T)).
\end{equation}
With constant step-size $\eta_t = \frac{1}{B}\sqrt{\frac{\radius}{T}}$ we have
\begin{equation}
    f(\ROUND(\overline{\gamma}_T)) - f(\gamma^*) \leq  \frac{17B}{8}\sqrt{\frac{\radius}{T}}.
\end{equation}
\end{thm}

\begin{proof}
The proof follows from \eqref{eq1}
\begin{equation}
    f(\overline{\gamma}_T) - f(\gamma^*) +2Bc(\overline{\gamma}_T) \leq  B\sqrt{\frac{\radius}{T}}\left(3+\log(T)\right),
\end{equation}
and
\begin{equation}
    f(\overline{\gamma}_T) - f(\gamma^*) +2Bc(\overline{\gamma}_T) \geq f(\ROUND(\overline{\gamma}_T)) - f(\gamma^*)
\end{equation}
by Lemma \ref{lem:round} and the fact that $f$ is $B$-Lipschitz.
\end{proof}

\begin{lem}\label{lem:main}
The iterates of algorithm \ref{alg:main} verify
\begin{equation}
    D_{\KL}(\gamma^*, \gamma_t) - D_{\KL}(\gamma^*, \gamma_{t+1}) = D_{\KL}(\gamma_{t+1}, \gamma_t) + \eta_t\langle \nabla f(\gamma_t), \gamma_{t+1}-\gamma^*\rangle,
\end{equation}
\end{lem}
\begin{proof}
We assume $t$ is odd, and write for any $\gamma\in\Delta_{m\times n}$:
\begin{align}\label{eq:anygamma}
    D_{\KL}(\gamma, \gamma_t) - D_{\KL}(\gamma, \gamma_{t+1}) &= \left\langle \gamma, \log\left(\gamma_{t+1}\oslash\gamma_{t}\right)\right\rangle,\\
    &= \left\langle \gamma, \log\left(\gamma_{t+1}\oslash\gamma'_{t+1}\right)\right\rangle + \left\langle \gamma, \log\left(\gamma'_{t+1}\oslash\gamma_{t}\right)\right\rangle,\\
     &= \left\langle \gamma,\log\left(\gamma_{t+1}\oslash\gamma'_{t+1}\right)\right\rangle- \eta_t\langle \nabla f(\gamma_t), \gamma\rangle,
\end{align}
we remark that
\begin{align}
   \left\langle \gamma^*, \log\left(\gamma_{t+1}\oslash\gamma'_{t+1}\right)\right\rangle
   &= \sum_{i}\sum_j \gamma^*_{ij}\log\left(\frac{\mu_i}{\sum_k(\gamma'_{t+1})_{ik}}\right),\\
   &= \left\langle \gamma^*1_n, \log\left(\mu\oslash(\gamma'_{t+1}1_n)\right)\right\rangle,\\
   &= \left\langle \mu, \log\left(\mu\oslash(\gamma'_{t+1}1_n)\right)\right\rangle,\\
\end{align}
and similarly
\begin{align}
    \left\langle \mu, \log\left(\mu\oslash(\gamma'_{t+1}1_n)\right)\right\rangle
    &=\left\langle \gamma_{t+1}1_n, \log\left(\mu\oslash(\gamma'_{t+1}1_n)\right)\right\rangle,\\
    &= \sum_{i}\sum_j (\gamma_{t+1})_{ij}\log\left(\frac{\mu_i}{\sum_k(\gamma'_{t+1})_{ik}}\right),\\
    &= \left\langle \gamma_{t+1}, \log\left(\gamma_{t+1}\oslash\gamma'_{t+1}\right)\right\rangle.
\end{align}
Thus for $\gamma = \gamma^*$ in \eqref{eq:anygamma}:
\begin{equation}
    D_{\KL}(\gamma^*, \gamma_t) - D_{\KL}(\gamma^*, \gamma_{t+1})= \left\langle \gamma_{t+1}, \log\left(\gamma_{t+1}\oslash\gamma'_{t+1}\right)\right\rangle - \eta_t\langle \nabla f(\gamma_t), \gamma^*\rangle.
\end{equation}
Now we consider $\gamma = \gamma_{t+1}$ in \eqref{eq:anygamma}:
\begin{equation}
    D_{\KL}(\gamma_{t+1}, \gamma_t) = \left\langle \gamma_{t+1}, \log\left(\gamma_{t+1}\oslash\gamma'_{t+1}\right)\right\rangle - \eta_t\langle \nabla f(\gamma_t), \gamma_{t+1}\rangle
\end{equation}
which allows to conclude
\begin{equation}
    D_{\KL}(\gamma^*, \gamma_t) - D_{\KL}(\gamma^*, \gamma_{t+1}) = D_{\KL}(\gamma_{t+1}, \gamma_t)+ \eta_t\langle \nabla f(\gamma_t), \gamma_{t+1}-\gamma^*\rangle.
\end{equation}
The reasoning is identical for $t$ even.
\end{proof}
\newpage
\begin{algorithm}[hbt!]
\caption{Stochastic Mirror Sinkhorn}\label{alg:stoch}
\KwData{Initialise $\gamma_1 =\overline{\gamma}_1= \mu\nu^\top$, define stepsize $\eta_t$.}
\For{$1\leq t\leq T$}{
Sample $g_t$ such that $\Esp{}{g_t\vert \gamma_t} = \nabla f(\gamma_t)$,\\
Update $\gamma_{t+1}' = \gamma_{t} \odot \exp\left(- \eta_tg_t\right),$\\
\eIf{$t$ is even}{
Rows: $\gamma_{t+1} = \Diag\left(\mu\oslash(\gamma_{t+1}'1_n)\right)\gamma_{t+1}'$,
}{
Cols: $\gamma_{t+1} = \gamma_{t+1}'\Diag\left(\nu\oslash((\gamma_{t+1}')^\top1_m)\right)$,
}
Update $\overline{\gamma}_{t+1} = \frac{t}{t+1}\overline{\gamma}_t + \frac{1}{t+1}\gamma_{t+1}$
}
\KwOut{$\overline{\gamma}_T = \frac{1}{T}\sum_{t=1}^T\gamma_t$ }

\end{algorithm}

\begin{lem}\label{lem:constraints}
The iterates of algorithm \ref{alg:main} verify
\begin{equation}
    D_{\KL}(\gamma_{t+1}, \gamma_t) \geq 2\Vert\gamma_{t+1}- \gamma_t\Vert_1^2\geq 2\max\{c(\gamma_t), c(\gamma_{t+1})\}.
\end{equation}
\end{lem}
\begin{proof}
The first inequality comes from Pinsker, the second from Jensen on the function $\Vert\cdot\Vert_1$:
\begin{equation}
    \Vert\gamma_{t+1}- \gamma_t\Vert_1 \geq \max\{\Vert\gamma_{t+1}1_n- \gamma_t 1_n\Vert_1,\Vert(\gamma_{t+1})^\top1_m- (\gamma_t)^\top 1_m \Vert_1\}.
\end{equation}
We assume $t$ is odd, then
\begin{equation}
    \Vert\gamma_{t+1}1_n- \gamma_t 1_n\Vert_1 = \Vert\gamma_{t+1}1_n- \mu\Vert_1 = c(\gamma_{t+1})
\end{equation}
and 
\begin{equation}
    \Vert(\gamma_{t+1})^\top1_m- (\gamma_t)^\top 1_m \Vert_1 = \Vert\nu - (\gamma_t)^\top 1_m \Vert_1 = c(\gamma_{t}).
\end{equation}
The proof is symmetrical for $t$ even.
\end{proof}


\begin{prop}\cite{altschuler2017near}\label{lem:round}
For all $\gamma\in\R_+^{m\times n}$
\begin{equation}
    \Vert \ROUND(\gamma) - \gamma\Vert_1\leq 2c(\gamma) = 2\left(\Vert \gamma 1_m - \mu\Vert_1 + \Vert \gamma^\top 1_n - \nu\Vert_1\right),
\end{equation}
moreover $ \ROUND(\gamma)\in\mathcal{T}(\mu, \nu)$ and the runtime of the algorithm is $O(mn).$
\end{prop}
\begin{prop}
Let $\ell\geq 0$. If $f$ is $\ell$-strongly convex with regards to the relative entropy, for any $\gamma\in\R^{m\times n}_+,$
\begin{equation}
    f(\gamma) - f(\gamma^*)\geq \ell D_{\KL}(\gamma^*, \gamma) -2\Vert f(\gamma^*)\Vert_\infty c(\gamma).
\end{equation}
This is also true for $\ell = 0$. If $f$ is $L$-smooth with regards to the relative entropy, then for any $\gamma\in\R^{m\times n}_+,$
\begin{equation}
    f(\gamma) - f(\gamma^*)\leq L D_{\KL}(\gamma^*, \gamma) +2\Vert f(\gamma^*)\Vert_\infty c(\gamma).
\end{equation}
\end{prop}

\begin{proof}
By strong convexity
\begin{equation}
    f(\gamma) - f(\gamma^*)\geq \ell D_{\KL}(\gamma^*, \gamma) + \langle \nabla f(\gamma^*), \gamma - \gamma^*\rangle
\end{equation}
Since $\gamma^*$ is the minimum of $f$ on $\mathcal{T}(\mu, \nu)$, 
\begin{equation}
    \langle \nabla f(\gamma^*), \gamma^*\rangle = \langle \nabla f(\gamma^*), \ROUND(\gamma)\rangle,
\end{equation}
so
\begin{equation}
    f(\gamma) - f(\gamma^*)\geq \ell D_{\KL}(\gamma^*, \gamma) -\Vert f(\gamma^*)\Vert_\infty \Vert \gamma - \ROUND(\gamma)\Vert_1
\end{equation}
and we conclude with Lemma \ref{lem:round}.

By smoothness,
\begin{equation}
    f(\gamma) - f(\gamma^*)\leq L D_{\KL}(\gamma^*, \gamma) + \langle \nabla f(\gamma^*), \gamma - \gamma^*\rangle
\end{equation}
so
\begin{equation}
    f(\gamma) - f(\gamma^*)\leq L D_{\KL}(\gamma^*, \gamma) +\Vert f(\gamma^*)\Vert_\infty \Vert \gamma - \ROUND(\gamma)\Vert_1
\end{equation}
and we also conclude with Lemma \ref{lem:round}.
\end{proof}

\subsection{Online Optimisation}



\begin{proof}[Proof of Theorem~\ref{thm:online}]
We follow the proof of Theorem~\ref{thm:main_appendix} up to \eqref{beforesum}
\begin{equation}
    0\leq \eta_t(f_t(\gamma_t) - f_t +2Bc(\gamma_t))\leq  D_{\KL}(\gamma^*, \gamma_t) - D_{\KL}(\gamma^*, \gamma_{t+1}) +\frac{9B^2\eta_t^2}{8}.
\end{equation}
Then we sum over $1\leq t\leq T$. Idem for the constraints.
\end{proof}

\subsection{Stochastic Case}


\begin{proof}[Proof of Theorem~\ref{thm:stoch}]
We follow the proof of Theorem~\ref{thm:main_appendix} up to \eqref{beforesum}, where everything is true in expectations,
\begin{equation}
    0\leq \Esp{}{\eta_t(f(\gamma_t) - f +2Bc(\gamma_t))}\leq  \Esp{}{D_{\KL}(\gamma^*, \gamma_t) - D_{\KL}(\gamma^*, \gamma_{t+1})} +\frac{9B_{\sigma}^2\eta_t^2}{8}.
\end{equation}
Then we note
\begin{equation}
    f(\ROUND(\gamma_t)) - f(\gamma^*)\leq f(\gamma_t) - f +2Bc(\gamma_t),
\end{equation}
we sum over $1\leq t\leq T$ and conclude with Jensen's inequality.
\end{proof}

\subsection{Optimal Transport}

\begin{algorithm}[hbt!]
\caption{Mirror Sinkhorn for Optimal Transport}\label{alg:ot}
\KwData{Initialise $\gamma_1 =\overline{\gamma}_1= \mu\nu^\top$, define stepsize $\eta_t$, streaming stochastic cost matrices $C_t$.}
\For{$1\leq t\leq T$}{
Update $\gamma_{t+1}' = \gamma_{t} \odot \exp\left(- \eta_tC_t\right),$\\
\eIf{$t$ is even}{
Rows: $\gamma_{t+1} = \Diag\left(\mu\oslash(\gamma_{t+1}'1_n)\right)\gamma_{t+1}'$,
}{
Cols: $\gamma_{t+1} = \gamma_{t+1}'\Diag\left(\nu\oslash((\gamma_{t+1}')^\top1_m)\right)$,
}
Update $\overline{\gamma}_{t+1} = \frac{t}{t+1}\overline{\gamma}_t + \frac{1}{t+1}\gamma_{t+1}$
}
\KwOut{$\overline{\gamma}_T = \frac{1}{T}\sum_{t=1}^T\gamma_t$ }

\end{algorithm}

\begin{thm}\label{thm:ot_appendix}
Let $(C_t)_{t\geq 1}$ be a sequence of random cost matrices such that $\Esp{}{C_t} = C$ for a given cost matrix $C\in\R^{m\times n}$ that verifies $\Vert C\Vert_\infty\leq 1$, and $\Esp{}{\Vert C_t - C\Vert_\infty^2}\leq \sigma^2$.
Let the stepsize be $\eta_t = \sqrt{\frac{\radius}{(1+\sigma^2)t}}$ in the algorithm \ref{alg:ot}, then the output $\overline{\gamma}_T$ after $T$ steps verifies:
\begin{equation}
    \Esp{}{\langle C, \overline{\gamma}_T\rangle} - \langle C, \gamma^*\rangle \leq  2\sqrt{\frac{(1+\sigma^2)\radius}{T}}(1+\log(T)),
\end{equation}
and the constraints are close to be verified
\begin{equation}
    \Esp{}{c(\overline{\gamma}_T)} \leq  \sqrt{\frac{\radius}{T}}(2+\log(T)),
\end{equation}
with
\begin{equation}
    c(\overline{\gamma}_T) = \Vert\overline{\gamma}_T1_n-\mu\Vert_1 + \Vert(\overline{\gamma}_T)^\top1_m-\nu\Vert_1.
\end{equation}
\end{thm}

\begin{proof}
This follows directly from Theorem~\ref{thm:main_appendix}.
\end{proof}


\begin{proof}[Proof of Theorem~\ref{thm:ot_round}]
This follows from Theorem~\ref{thm:ot_appendix} with the same proof as Theorem~\ref{thm:main_round_appendix}.
\end{proof}
\begin{proof}[Proof of Theorem~\ref{thm:ot_round_fixedstep}]
This follows directly from the proofs of Theorem~\ref{thm:ot_appendix} and of Theorem~\ref{thm:main_round_appendix}.
\end{proof}

\subsection{Strong Convexity}

\begin{proof}[Proof of Theorem~\ref{thm:strong}]
By smoothness,
\begin{equation}
    f(\gamma_{t+1}) - f(\gamma_{t})\leq \langle \nabla f(\gamma_t), \gamma_{t+1} - \gamma_{t}\rangle + L D_{\KL}(\gamma_{t+1}, \gamma_t)
\end{equation}
by strong convexity,
\begin{equation}
    f(\gamma_{t}) - f(\gamma^*)
    \leq \langle\nabla f(\gamma_t), \gamma_{t} - \gamma^*\rangle - \ell D_{\KL}( \gamma^*,\gamma_{t})
\end{equation}
adding the two with Lemma \ref{lem:main},
\begin{equation}
    (1-\ell\eta_t)D_{\KL}(\gamma^*, \gamma_t) - D_{\KL}(\gamma^*, \gamma_{t+1}) \geq (1-L \eta_t)D_{\KL}(\gamma_{t+1}, \gamma_t) + \eta_t(f(\gamma_{t+1}) - f(\gamma^*)).
\end{equation}

We reason as in the proof of Theorem \ref{thm:main_appendix}, assuming $\eta_t<1/L$:
\begin{equation}
    \eta_T\left(f(\gamma_{t+1})-f(\gamma^*)+B'c(\gamma_{t+1})\right)\leq (1-\ell\eta_t)D_{\KL}(\gamma^*, \gamma_t) - D_{\KL}(\gamma^*, \gamma_{t+1}) + \frac{(B')^2\eta_t^2}{8(1-L\eta_t)}.
\end{equation}
Let $\eta_t = \frac{1}{\ell t}$, then
\begin{equation}
f(\gamma_{t+1})-f(\gamma^*)+B'c(\gamma_{t+1})\leq \ell(t-1)D_{\KL}(\gamma^*, \gamma_t) - \ell tD_{\KL}(\gamma^*, \gamma_{t+1}) + \frac{(B'+L)^2}{8\ell t}.
\end{equation}

A sum up to $T$ provides
\begin{equation}
    \sum_{t=1}^T(f(\gamma_{t+1}) - f(\gamma^*)+B'\Vert\gamma_{t+1}- \gamma_t\Vert_1) + \ell TD_{\KL}(\gamma^*, \gamma_{T+1})\leq \sum_{t=1}^T\frac{(B'+L)^2}{8\ell t}.
\end{equation}
Finally we use Jensen and Lemma \ref{lem:round} with $B'=2B$.
\end{proof}

\subsection{Tensor Problem}

\begin{algorithm}[hbt!]
\caption{Tensor Mirror Sinkhorn}\label{alg:tensor}
\KwData{Initialise $(\gamma_1)_{i_1\dots i_d} =(\overline{\gamma}_1)_{i_1\dots i_d}= \mu_{i_1}\cdots \mu_{i_d}$, define stepsize $\eta_t$.}
\For{$1\leq t\leq T$}{
Update $\gamma_{t+1}' = \gamma_{t} \odot \exp\left(- \eta_t\nabla f(\gamma_t)\right),$\\
\For{$1\leq k \leq d$}{
Sum across all dimensions but $k$: $$(S_k(\gamma_{t+1}'))_{i_k} = \sum_{j: j_k = i_k}(\gamma_{t+1}')_{j_1,\dots,j_d}.$$\\
Constraint distance $c_k = D_{\KL{}}(\mu_k, S_k(\gamma'_{t+1}))$
}
Find furthest dimension $K = \argmax_k c_k$,\\
Normalize along dimension $K$: $$(\gamma_{t+1})_{i_1\cdots i_d} = \frac{(\mu_{K})_{i_{K}}}{\left(S_{K}(\gamma_{t+1}')\right)_{i_{K}}}(\gamma_{t+1}')_{i_1\cdots i_d}.$$}
\KwOut{$\overline{\gamma}_T = \frac{1}{T}\sum_{t=1}^T\gamma_t$ }

\end{algorithm}
 

\begin{proof}[Proof of Theorem~\ref{thm:tensor}]
The proof of Lemma \ref{lem:main} is still valid for the tensor case. Lemma \ref{lem:constraints} is modified with cyclical Bregman projections on the marginal spaces:
\begin{equation}
    \sum_{t=1}^TD_{\KL}(\mu_{k(t)}, S_{k(t)}(\gamma_t))\leq \radius + 2\sum_{t=1}^T\eta_t\Vert\nabla f(\gamma_t)\Vert_\infty
\end{equation}
with $k(t)$ the dimension chosen at iteration $t$. Then, since the choice of dimension to normalise is greedy,
\begin{equation}
    \sum_{t=1}^T\max_{1\leq k\leq d}D_{\KL}(\mu_{k}, S_{k}(\gamma_t))\leq \radius + 2\sum_{t=1}^T\eta_t\Vert\nabla f(\gamma_t)\Vert_\infty,
\end{equation}
and we conclude with Jensen's inequality.
\end{proof}

\section{Complementary experimental results}
\label{app:exp}
\begin{figure*}[!ht]
    \centering
    \includegraphics[width=\textwidth]{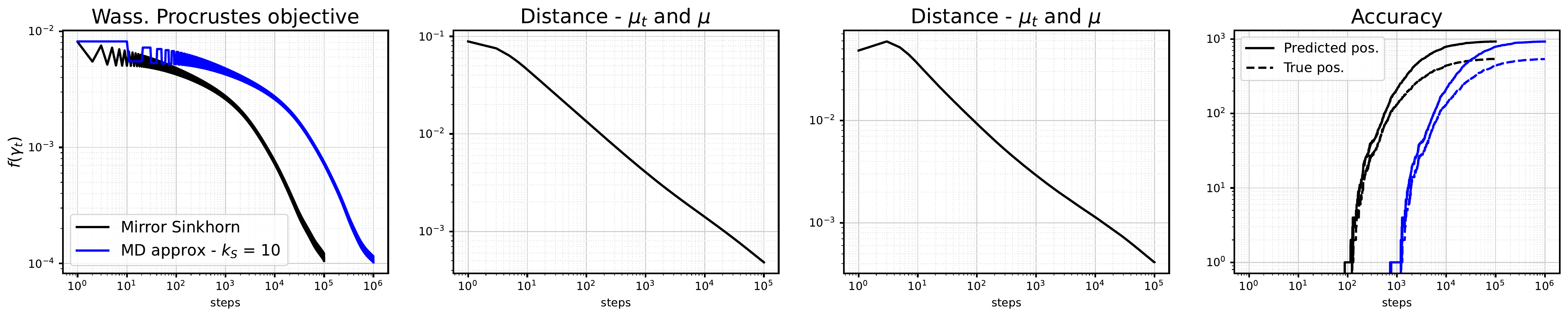}
    \caption{Performance of the Mirror Sinkhorn algorithm (black) on the Wasserstein Procrustes objective, for the experiment described in Section~\ref{sec:expe-wp}. This is compared with an approximation of mirror descent (blue) with $k_S = 10$ steps of Sinkhorn projection at each gradient update. \textbf{Left:} The value of $f(\gamma_t)$ as a function of $t$. \textbf{Centers:} The $\ell_1$ distance between the two marginals $(\mu_t, \nu_t)$ of $\gamma_t$ and constraints $(\mu, \nu)$ at any given time. \textbf{Right}: For the predicted assignment matrix, thresholded version of $\gamma_t$, the number (compared to $n=1,047$) of predicted positives (solid line) and of true positives (dashed line) for both algorithms.}
    \label{fig:wp-graphs-app}
\end{figure*}

\begin{figure*}[!ht]
    \centering
    \includegraphics[width=\textwidth]{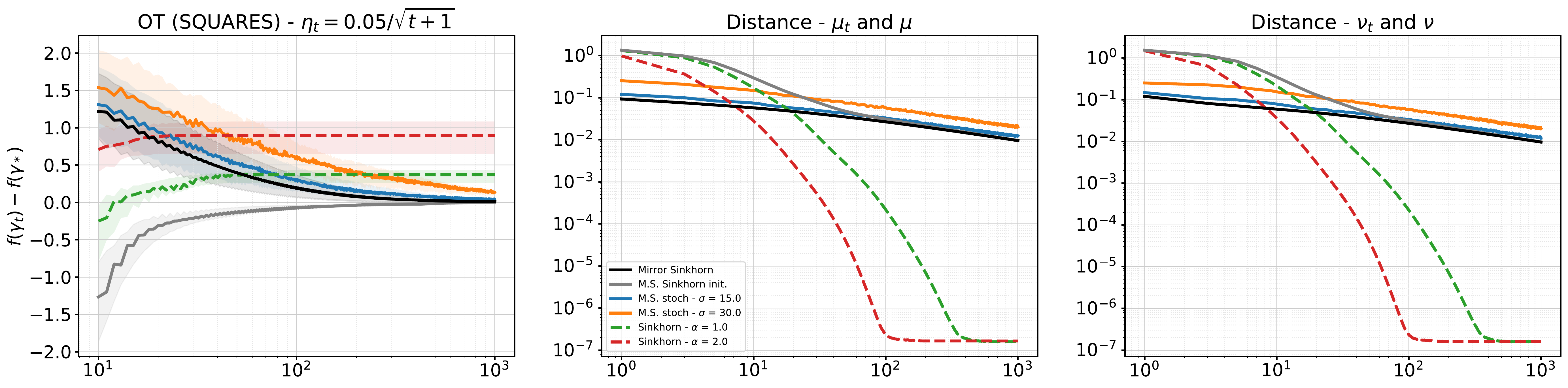}
        \vspace{-0.7cm}
    \caption{Performance of Mirror Sinkhorn on OT for SQUARES data \citep{altschuler2017near}. See Figure~\ref{fig:ot-graphs} and Section~\ref{sec:ot-expe} for details.}
    \label{fig:ot-app-graphs}
\end{figure*}

\end{document}